\documentclass[conference]{IEEEtran}
\IEEEoverridecommandlockouts
\usepackage{cite}
\usepackage{amsmath,amssymb,amsfonts}
\usepackage{algorithmic}
\usepackage{graphicx}
\usepackage{textcomp}
\usepackage{xcolor}
\def\BibTeX{{\rm B\kern-.05em{\sc i\kern-.025em b}\kern-.08em
    T\kern-.1667em\lower.7ex\hbox{E}\kern-.125emX}}

\usepackage{microtype}
\usepackage{graphicx}
\usepackage{booktabs} 
\usepackage{multirow}
\usepackage{multicol}
\usepackage{makecell}
\usepackage{tablefootnote}
\usepackage[table]{xcolor}
\usepackage{fix-cm} 
\DeclareMathSizes{10}{10}{6}{4}

\usepackage[hidelinks]{hyperref}

\usepackage{amsmath}
\usepackage{amssymb}
\usepackage{mathtools}
\usepackage{amsthm}
\usepackage{caption}
\usepackage{subcaption}
\usepackage{booktabs}
\usepackage[flushleft]{threeparttable}
\usepackage{MnSymbol}
\usepackage{lipsum}

\usepackage{bigstrut}
\usepackage{colortbl}
\usepackage{wrapfig}

\usepackage{enumitem}
\setlist[itemize]{noitemsep}
\setlist[enumerate]{noitemsep}

\usepackage[ruled, vlined, linesnumbered]{algorithm2e} 

\usepackage[capitalize,noabbrev]{cleveref}
\makeatletter
\AddToHook{cmd/appendix/before}{%
    \crefalias{section}{appendix}%
    \crefalias{subsection}{appendix}
}
\makeatother

\makeatletter
\newtheorem*{rep@theorem}{\rep@title}
\newcommand{\newreptheorem}[2]{%
\newenvironment{rep#1}[1]{%
 \def\rep@title{#2 \ref{##1}}%
 \begin{rep@theorem}}%
 {\end{rep@theorem}}}
\makeatother

\theoremstyle{plain}
\newtheorem{theorem}{Theorem}[section]
\newreptheorem{theorem}{Theorem}
\newtheorem{proposition}[theorem]{Proposition}
\newreptheorem{proposition}{Proposition}
\newtheorem{lemma}[theorem]{Lemma}
\newreptheorem{lemma}{Lemma}

\newreptheorem{corollary}{Corollary}
\theoremstyle{definition}
\newtheorem{definition}[theorem]{Definition}

\theoremstyle{remark}
\newtheorem{remark}[theorem]{Remark}

\usepackage{amsmath}
\usepackage{amssymb}
\usepackage[normalem]{ulem}
\usepackage{bm}
\usepackage{bbm}
\usepackage{caption}
\usepackage{subcaption}
\usepackage{adjustbox}

\usepackage{wrapfig}

\usepackage[flushleft]{threeparttable}

\usepackage{tikz}
\usetikzlibrary{shapes.geometric, arrows}

\usepackage{bigstrut}
\usepackage{colortbl}
\usepackage{wrapfig}

\usepackage{enumitem}
\setlist[itemize]{noitemsep}
\setlist[enumerate]{noitemsep}

\usepackage{mathtools}
\usepackage{booktabs}
\usepackage{multirow}
\usepackage{multicol}

\usepackage[ruled, vlined, linesnumbered]{algorithm2e} 

\newcommand\bolden[1]{{\boldmath\bfseries#1}}

\usepackage{contour}
\contourlength{0.8pt}

\newcommand{\myuline}[1]{%
  \uline{\phantom{#1}}%
  \llap{\contour{white}{#1}}%
}

\newcommand{\smallsection}[1]{{\noindent {\bolden{\myuline{#1}}}}}

\definecolor{lucky}{RGB}{120, 120, 120}

\newcommand\blue[1]{\textcolor{blue}{#1}}

\usepackage{graphicx}
\usepackage{xspace}
\usepackage{amsthm}
\usepackage{amsfonts}

\usepackage[capitalize,noabbrev]{cleveref}

\theoremstyle{plain}
\newtheorem{property}[theorem]{Property}

\def\mydefbb#1{\expandafter\def\csname bb#1\endcsname{\ensuremath{\mathbb{#1}}}}
\def\mydefallbb#1{\ifx#1\mydefallbb\else\mydefbb#1\expandafter\mydefallbb\fi}
\mydefallbb ABCDEFGHIJKLMNOPQRSTUVWXYZ\mydefallbb

\def\mydefcal#1{\expandafter\def\csname cal#1\endcsname{\ensuremath{\mathcal{#1}}}}
\def\mydefallcal#1{\ifx#1\mydefallcal\else\mydefcal#1\expandafter\mydefallcal\fi}
\mydefallcal ABCDEFGHIJKLMNOPQRSTUVWXYZ\mydefallcal

\newcommand{\abs}[1]{\left|#1\right|}
\newcommand{\set}[1]{\{#1\}}

\newcommand{\overlap}{\operatorname{Ov}}

\newcommand{\feigm}{f^{\text{EI}}_{p}}

\newcommand{\flb}{f^{\text{LB}}_{p; g, R}}
\newcommand{\fiid}{f^{\text{PB}}_{p; g, R}}
\newcommand{\fxx}{f^{\text{X}}_{p; g, R}}
\newcommand{\pfull}{p \colon \binom{V}{2} \to [0, 1]}
\newcommand{\gfull}{g \colon V \to [0, 1]}

\usepackage{enumitem}
\setlist[itemize]{leftmargin=*, itemsep=0pt}
\setlist[enumerate]{leftmargin=*, itemsep=0pt}

\newcommand{\GT}{\textsc{GroundT}\xspace}
\newcommand{\EI}{\textsc{EdgeInd}\xspace}
\newcommand{\LB}{\textsc{LoclBdg}\xspace}
\newcommand{\PB}{\textsc{ParaBdg}\xspace}
\newcommand{\hamster}{\textit{Hams}\xspace}
\newcommand{\facebook}{\textit{Fcbk}\xspace}
\newcommand{\polblog}{\textit{Polb}\xspace}
\newcommand{\spam}{\textit{Spam}\xspace}
\newcommand{\biopg}{\textit{Cepg}\xspace}
\newcommand{\bioht}{\textit{Scht}\xspace}

\usepackage{etoolbox}

\makeatletter
\patchcmd{\@algocf@start}
  {-1.5em}
  {0pt}
  {}{}
\makeatother

\usepackage{tikz}
\usetikzlibrary{shapes.geometric, arrows}

\let\citep\cite
\let\citet\cite

\begin{document}

\title{Edge Probability Graph Models Beyond Edge Independency: Concepts, Analyses, and Algorithms
}

\author{
Fanchen Bu\textsuperscript{1},
Ruochen Yang\textsuperscript{2},
Paul Bogdan\textsuperscript{2},
and Kijung Shin\textsuperscript{1}\\
\textsuperscript{1}KAIST, Republic of Korea and \textsuperscript{2}University of Southern California, United States\\
\textsuperscript{1}\{boqvezen97, kijungs\}@kaist.ac.kr and
\textsuperscript{2}\{ruocheny, pbogdan\}@usc.edu
}

\maketitle

\nocite{appendix}

\begin{abstract}
Desirable random graph models (RGMs) should 
\textit{(i)} reproduce \textit{common patterns} in real-world graphs (e.g., power-law degrees, small diameters, and high clustering),
\textit{(ii)} generate \textit{variable} (i.e., not overly similar) graphs,
and \textit{(iii)} remain \textit{tractable} to compute and control graph statistics.
A common class of RGMs (e.g., Erd\H{o}s-R\'{e}nyi and stochastic Kronecker) outputs edge probabilities, so we need to \textit{realize} (i.e., sample from) the output edge probabilities to generate graphs.
Typically, the existence of each edge is assumed to be determined independently, for simplicity and tractability.
However, with edge independency, RGMs provably cannot produce high subgraph densities and high output variability simultaneously.
In this work, we explore RGMs beyond edge independence that can better reproduce common patterns
while maintaining high tractability and variability.
Theoretically, we propose an edge-dependent realization (i.e., sampling) framework called \textit{binding} that provably preserves output variability, and derive \textit{closed-form} tractability results on subgraph (e.g., triangle) densities.
Practically, we propose algorithms for graph generation with binding and parameter fitting of binding.
Our empirical results demonstrate that RGMs with binding exhibit high tractability and well reproduce common patterns, significantly improving upon edge-independent RGMs.

\end{abstract}


\section{Introduction}\label{sec:intro}

Random graph models (RGMs) help us understand, analyze, and predict real-world systems~\citep{drobyshevskiy2019random},
{with various practical applications, e.g., graph algorithm testing~\citep{murphy2010introducing}, statistical testing~\citep{ghoshdastidar2017two}, and graph anonymization~\citep{backstrom2007wherefore}.}
Desirable RGMs should generate graphs with 
\textit{common patterns} in real-world graphs,
such as high clustering,\footnote{\textit{High clustering} means high subgraph densities as in, e.g., \citet{newman2003properties} and \citet{pfeiffer2012fast}.} power-law degrees, and small diameters~\citep{chakrabarti2006graph}.
At the same time, the generated graphs should be \textit{variable}, i.e., not highly-similar or even near-identical, and the RGMs should be \textit{tractable}, i.e., we can compute and control graph statistics of the generated random graphs.

Many RGMs output individual edge probabilities and generate graphs accordingly, e.g., 
the Erd\H{o}s-R\'{e}nyi model~\citep{erdos1959randomgraph1}, 
the Chung-Lu model~\citep{chung2002connected}, 
the stochastic block model~\citep{holland1983stochastic}, and
the stochastic Kronecker model~\citep{leskovec2010kronecker}.
To generate graphs from edge probabilities, we need \textit{realizations} (i.e., sampling), where edge independency (i.e., the existences of edges are determined independently) is widely assumed for simplicity and tractability.
Although edge-independent RGMs have high tractability and may reproduce some patterns (e.g., power-law degrees and small diameters),
they empirically fail to preserve some other patterns, especially high clustering~\citep{moreno2018tied,seshadhri2013depth}, which has also been theoretically validated.
Specifically, edge-independent RGMs provably cannot generate graphs with high triangle density and high output variability at the same time~\citep{nips21powerOfEIGMs}.

Naturally, we ask:
Can we apply realization \textit{without assuming edge independency} so that we can improve upon such RGMs to generate graphs \textit{reproducing common patterns} and \textit{having high variability}, while still \textit{ensuring high tractability}?
To address this question, we propose and explore the concept of \textit{edge probability graph models} (EPGMs), i.e., RGMs that are still based on edge probabilities but do not assume edge independency, from theoretical and practical perspectives.
Our key message is a positive answer to the question.
Specifically, our novel contributions are fourfold:
\begin{enumerate}[noitemsep,topsep=0mm]
    \item \textbf{Concepts (\cref{sec:prob_statement}):} We formally propose and define the \textit{concept} of EPGMs,
    and theoretically show some basic \textit{properties} of EPGMs, e.g., even with edge dependency introduced, {the \textit{output variability remains as high}} as the corresponding edge-independent model.
    \item \textbf{Analyses (\cref{sec:binding}):} We propose \textit{binding}, a \textit{(i) pattern-reproducing}, \textit{(ii) tractable}, and \textit{(iii) flexible} realization scheme, to construct EPGMs with different levels of edge dependency, and derive tractability results on the \textit{closed-form} subgraph (e.g., triangle) densities.
    \item \textbf{Algorithms (\cref{sec:binding}):} We propose practical algorithms for
    \textit{(i) graph generation} with binding, and for 
    \textit{(ii) efficient parameter fitting} to control the graph statistics generated by EPGMs with binding.
    \item \textbf{Experiments (\cref{sec:exp}):}     
    Via experiments on real-world graphs, we show the power of the proposed binding and parameter-fitting algorithms to reproduce common graph patterns and their efficiency, validating the correctness of our theoretical analyses and practical algorithms.
\end{enumerate}

\smallsection{Reproducibility.}
The appendix, code, and data are publicly available online in~\citep{appendix} (\url{https://github.com/bokveizen/edge-probability-graph-model}).

\section{Preliminaries}\label{sec:prelim_background}

\smallsection{Graphs.}
A \textit{graph} $G = (V, E)$ is defined by a \textit{node set} $V = V(G)$ and an \textit{edge set} {$E = E(G) \subseteq \binom{V}{2} \coloneqq \set{V' \subseteq V \colon \abs{V'} = 2}$}.\footnote{In this work, we consider undirected unweighted graphs without self-loops following common settings for random graph models. {See Appendix~C-A~\cite{appendix} for discussions on more general graphs, e.g., directed or weighted graphs.}}
For a node $v \in V$, the set of its \textit{neighbors} is $N(v; G) = \set{u \in V \colon (u, v) \in E(G)}$, and its \textit{degree} $d(v; G) = \abs{N(v; G)}$ is the number of its neighbors.
Given $V' \subseteq V$, the \textit{induced subgraph} of $G$ on $V'$ is $G[V'] = (V', E \cap \binom{V'}{2})$.

\smallsection{Triangles and clustering.}
Given $G = (V, E)$, the number of wedges (i.e., open triangles) is $n_{w}(G) = \sum_{v \in V} \binom{d(v)}{2}$. The \textit{global clustering coefficient} (GCC) of $G$ is defined as $\text{GCC}(G) = \frac{\triangle(G)}{n_w(G)}$,
where $\triangle(G)$ is the number of triangles in $G$.
The \textit{average local clustering coefficient} (ALCC) of $G$ is defined as $\text{ALCC}(G) = \sum_{v \colon d(v) \geq 2} \frac{\triangle(v; G)}{\binom{d(v)}{2}}$, where $\triangle(v; G)$ is the number of triangles involving $v$ in $G$.

\smallsection{Random graph models (RGMs).}
Fix a node set $V = [n] = \set{1, 2, \ldots, n}$ with $n \in \bbN$.
Let $\calG(V) = \set{G = (V, E) \colon E \subseteq \binom{V}{2}}$ denote the set of all $2^{\binom{n}{2}}$ possible {node-labeled} (i.e., each node is assigned a distinct identifier $i \in [n]$) graphs on $V$.
    A \textit{random graph model} (RGM) is defined as a probability distribution $f: \calG(V) \to [0, 1]$ with $\sum_{G \in \calG(V)} f(G) = 1$.
    For each graph $G \in \calG(V)$, $f(G)$ is the probability of $G$ being generated by the RGM $f$.
    For each node pair $(u, v)$ with $u, v \in V$, the (marginal) \textit{edge probability} of $(u, v)$ under the RGM $f$ is $\Pr_{f}[(u, v)] \coloneqq \sum_{G \in \calG(V)} f(G) \mathbf{1}[(u, v) \in E(G)]$, where $\mathbf{1}[\cdot]$ is the indicator function.

\smallsection{Edge independent graph models (EIGMs).}
Given edge probabilities, edge independency is widely assumed in many existing RGMs, resulting in the concept of edge independent graph models (EIGMs~\citep{nips21powerOfEIGMs}; also known as \textit{inhomogeneous Erd\H{o}s-R\'{e}nyi graphs}~\citep{klopp2017oracle}).
The EIGM $\feigm$ w.r.t. edge probabilities $\pfull$ is defined by
$\feigm(G) = \prod_{(u, v) \in E(G)} p(u, v) \prod_{(u', v') \notin E(G)} (1 - p(u', v')), \forall G \in \calG(V)$.

\section{Related work and background}\label{sec:rel_wk}

\subsection{Limitations of EIGMs}\label{sec:rel_wk:limit_eigms}
Chanpuriya et al. \cite{nips21powerOfEIGMs} defined the concept of \textit{overlap} to measure the variability of RGMs, where a high overlap value implies low variability.
The overlap of an RGM $f$ is the expected proportion of edges in two generated graphs: $\overlap(f) = \frac{\bbE_{f} \abs{E(G') \cap E(G'')}}{\bbE_{f} \abs{E(G)}}$, where $G$, $G'$, and $G''$ are three random graphs independently generated by $f$.

High variability (i.e., low overlap) is important for RGMs~\citep{de2018molgan}, {as generating overly similar graphs undermines the effectiveness of RGMs in common applications, e.g., graph algorithm testing and graph anonymization (see \cref{sec:intro}).}
Chanpuriya et al. \citet{nips21powerOfEIGMs} showed that EIGMs cannot generate graphs with high clustering, which is a common property in real-world graphs~\citep{chakrabarti2006graph},
and low overlap at the same time.

Chanpuriya et al. \citet{arxiv23edgeDpendency} recently extended their theoretical results to edge-dependent RGMs.
However, they did not
provide practical graph generation algorithms or detailed tractability results,\footnote{Their graph generation relies on time-consuming maximal clique enumeration. See Appendix~C-B~\cite{appendix} for more discussions.}
while tractability results and practical graph generation are part of our focus in this work.

{Some methods shift edge probabilities from existing EIGMs by accept-reject~\citep{mussmann2015incorporating} or mixing different 
existing EIGMs~\cite{kolda2014scalable}.
Such methods are essentially still EIGMs, and inevitably have high overlap (i.e., low variability).
See Appendix~D-F~\cite{appendix} for more discussion and evaluation on such methods.}

\subsection{Edge-dependent RGMs}\label{sec:rel_wk:dependency}
Edge dependency exists in various RGMs, e.g.,
preferential attachment models~\citep{barabasi1999emergence}, 
small-world graphs~\citep{watts1998collective},
random geometric graphs~\citep{penrose2003random},
and exponential random graphs~\citep{lusher2013exponential}.
Recent work on exchangeable network models, which initially assume symmetry among nodes, introduce \textit{node asymmetry} to improve expressiveness~\citep{crane2018edge,wu2025tractably}.
In a similar spirit but from a different perspective, we introduce \textit{edge dependence} upon EIGMs to enhance the power of RGMs.

We differ from existing edge-dependent RGMs as follows:
\begin{itemize}[noitemsep,topsep=0pt]
    \item We provide a novel perspective on constructing edge-dependent RGMs by decomposing RGMs into (1) the marginal edge probabilities and (2) the realization of the probabilities (see \cref{def:EPGMs}).
    \item This new perspective allows us to improve upon existing EIGMs, inheriting their strengths (e.g., variability and tractability; see \cref{thm:constant_degrees_overlap} and \cref{thm:local_binding_motif_probs})
    and alleviating their weaknesses (e.g., inability to generate high clustering; see \cref{sec:app:p1_clustering}).
    \item We derive \textit{closed-form} tractability results on subgraph densities     
    (see \cref{thm:local_binding_motif_probs}), 
    while only \textit{asymptotic} results, as the number of nodes goes to infinity, are usually available for existing models~\citep{ostroumova2017general,gu2013clustering,bhat2016densification}.
    \item Our model empirically shows better capability in generating graphs with both high clustering and high variability than existing models (see \cref{sec:compare_other_models}).
\end{itemize}

\smallsection{Target.} In this work, we aim to propose novel edge-dependent RGMs with the following desirable properties:
\begin{itemize}[noitemsep,topsep=0pt]
    \item \textit{Reproducing common patterns} widely observed in real-world graphs, e.g., high clustering, power-law degrees, and small diameters~\citep{chakrabarti2006graph}.
    \item \textit{Showing high output variability}, generating variable graphs with low overlap (see \cref{sec:rel_wk:limit_eigms}).
    \item \textit{Having high tractability}, with the feasibility of obtaining closed-form results of graph statistics in generated graphs.
\end{itemize}

\section{Edge probability graph models (EPGMs): Concept and basic properties}\label{sec:prob_statement}
In this section, we introduce the concept of edge probability graph models (EPGMs) and show some basic properties of EPGMs.
EIGMs generate graphs assuming edge independency.
In contrast, we explore a broader class of edge probability graph models (EPGMs) going beyond edge independency.
\begin{definition}[EPGMs]\label{def:EPGMs}
    Given edge probabilities $\pfull$, the set $\calF(p)$ of \textit{edge probability graph models} (EPGMs) w.r.t. $p$ consists of all the RGMs with marginal edge probabilities $p$, i.e.,
    $\calF(p) \coloneqq \set{f \colon \Pr_f[(u, v)] = p(u, v), \forall u, v \in V}$.
\end{definition}

By EPGMs, we decompose each RGM into two factors:
\textbf{(F1)} the marginal probability of each edge and
\textbf{(F2)} how the edge probabilities are realized (i.e., {sampled}).
EIGMs have overlooked (F2), simply using realization assuming edge independency.
Our decomposition via EPGMs introduces a novel perspective on imposing edge dependency.

Below, we show some basic properties of EPGMs and discuss their implication and significance.

\begin{property}[EPGMs are general]\label{thm:EPGMs_general}
    {Any RGM can be represented as an EPGM (w.r.t its marginal edge probabilities).}
\end{property}
\begin{proof}[Proof sketch]
Any RGM is a multivariate distribution that can be represented by the marginal probabilities of each variable and the dependency among them.
\end{proof}
Note that, for all theoretical results in the paper, we provide \textbf{formal statements and/or full proofs} in Appendix~A~\cite{appendix}.

\smallsection{Implication.}
\cref{thm:EPGMs_general} shows the \textit{generality} of EPGMs, as they encompass all possible RGMs. 
However, it also implies that the space of EPGMs is too vast to be fully explored,
{which motivates us to find \textit{good subsets} of EPGMs.}

Finally, since each individual edge probability is preserved as in $p$, several other quantities are consequently preserved.
\begin{property}[EPGMs have constant expected degrees and overlap]\label{thm:constant_degrees_overlap}
    For any $\pfull$, the expected node degrees and overlap (see \cref{sec:rel_wk:limit_eigms}) of all the EPGMs w.r.t. $p$ are constant (i.e., depend only on $p$), ensuring EPGMs preserve degree distributions in expectation and have as low overlap (i.e., high output variability) as the corresponding EIGM.
\end{property}
\begin{proof}[Proof sketch]
It follows from the linearity of expectation.
\end{proof}

\smallsection{Implications.}
Many EIGMs can generate graphs with desirable degree distributions, and \cref{thm:constant_degrees_overlap} ensures that EPGMs inherit such strengths (see \cref{fig:main_fig} for empirical evidence).
Moreover, it ensures that EPGMs have as high output variability as the corresponding EIGM, which is an important and desirable property of RGMs (see \cref{sec:rel_wk}).

\smallsection{Research questions.}
Inspired by the basic properties of EPGMs, below, we explore EPGMs from both theoretical and practical perspectives, aiming to answer two questions:
\begin{itemize}[itemsep=0pt]
    \item \textbf{(RQ1; Theory)} {What good subsets of EPGMs are pattern-reproducing, flexible, and tractable?}
    \item \textbf{(RQ2; Practice)} How to generate graphs using such EPGMs and fit the parameters of EPGMs?
\end{itemize}

\normalem
\begin{algorithm}[t!]
    \small
    \caption{General Binding}\label{alg:binding_general}
    \DontPrintSemicolon
    \SetKwInOut{Input}{Input}
    \SetKwInOut{Output}{Output}
    \SetKw{Continue}{continue}
    \SetKwComment{Comment}{\blue{$\triangleright$ }}{}
    \SetKwFunction{proc}{binding}
    \SetKwProg{myproc}{Procedure}{}{}
    
    \Input{        
        (1) edge probabilities $p \colon \binom{V}{2} \to [0, 1]$; (2) pair partition $\calP$ s.t. $\bigcup_{P \in \calP} P = \binom{V}{2}$ with $P \cap P' = \emptyset, \forall P \neq P' \in \calP$
    }
    \Output{generated graph $G$}

    $E \gets \bigcup_{P \in \calP} \text{\proc{$p$, $P$}}$ \\  \label{alg:binding_general:update_E} 
    \Return{$G = (V, E)$} \label{alg:binding_general:return_G}
    
    \myproc{\proc{$\hat{p}$, $\hat{P}$}}{\label{alg:binding_general:local_binding}
        sample a random variable $s \sim \calU(0, 1)$ \\ \label{alg:binding_general:local_binding:sample_s}
        $\hat{E} \gets \set{(u, v): \hat{p}(u, v) \geq s}$ 
        \hfill \Comment{\blue{Decided by same $s$}}\label{alg:binding_general:local_binding:add_edge}
        \Return{$\hat{E}$} \\
        \label{alg:binding_general:local_binding:return_E} 
    }
\end{algorithm}
\ULforem

\section{Binding: Analyses and algorithms}\label{sec:binding}

The output variability of EPGMs is guaranteed by \cref{thm:constant_degrees_overlap}.
Therefore, we now aim to find good subsets of EPGMs for given edge probabilities (which can be obtained by EIGMs) that 
\textit{(i)} reproduce \textit{common patterns} 
(especially high clustering, the bottleneck of EIGMs) 
and
\textit{(ii)} are \textit{tractable} (i.e., with controllable graph statistics).

\subsection{General binding: General mathematical framework}\label{subsec:binding:binding}
Desirable RGMs should generate graphs with common patterns, e.g., high clustering, power-law degrees, and small diameters.
Some EIGMs (e.g., stochastic Kronecker) already do well in the latter two aspects, so we focus on the bottleneck ``high clustering'' and seek EPGMs that achieve it.\footnote{{Empirically, binding maintains or even improves the generation quality w.r.t. several different graph metrics, including but not limited to degrees and diameters (see \cref{subsec:EPGMs_do_not_harm_others}).}}
To this end, we first study and propose \textit{binding}, {a general mathematical framework} that introduces positive dependency among edges to determine multiple edge existences together.
Binding can be seen as a way of probability coupling~\cite{thorisson1995coupling}, 
with given marginal distributions (i.e., edge probabilities).
The process of binding is described in \cref{alg:binding_general},
where edge dependence is imposed in each group of pairs, in the sense that, if a node pair is sampled as an edge, all the pairs with higher probabilities must be sampled (Line~\ref{alg:binding_general:local_binding:add_edge}).
{{Note that, {\cref{alg:binding_general} describes a general framework, while} our practical algorithms (\cref{alg:local_binding,alg:iid_binding}) do not need to choose an explicit partition $\calP$ beforehand.}}
See Definition~A.4 in Appendix~A~\cite{appendix} for a mathematical definition of binding.

There are two basic properties of binding:
\textit{(i)} binding is \textit{correct}, i.e., generates EPGMs, and
\textit{(ii)} binding produces higher subgraph densities (i.e., higher clustering) than EIGMs, which aligns with our motivation to improve clustering.

\begin{proposition}\label{thm:binding_preserve_probs}  \cref{alg:binding_general} with input $p$ (and any $\calP$) correctly produces an EPGM w.r.t. $p$.
\end{proposition}
\begin{proof}[Proof sketch]
    Each edge $(u, v)$ exists when $p(u, v) \geq s$, which happens with probability $p(u, v)$ since $s \sim \calU(0,1)$.
\end{proof}

\begin{proposition}\label{thm:binding_higher_density}
    Binding produces higher or equal subgraph densities, compared to the corresponding EIGMs.
\end{proposition}
\begin{proof}[Proof sketch]
    When multiple edges are grouped together, the probability that all of them exist increases from the \textit{product} of their marginal probabilities to the \textit{minimum} of them. 
\end{proof}

With binding, we can construct EPGMs with different levels of edge dependency by different ways of binding the node pairs.
Let us first study two extreme cases.

\smallsection{Minimal binding.}
{EIGMs are the case with minimal binding, i.e., without binding, where the partition contains only sets of a single pair, i.e., $\calP = \set{\set{(u, v)} \colon u, v \in V}$.}

\smallsection{Maximal binding.}
Maximal binding corresponds to the case with $\calP = \set{\binom{V}{2}}$, i.e., all the pairs are bound together.
It achieves the upper bound of subgraph densities.

Below, we explore EPGMs between the two extremes.

\subsection{Local binding: Practical and tractable scheme}\label{sec:binding:local_binding}
{Building upon the general framework introduced in \cref{subsec:binding:binding}, we propose practical binding algorithms.}
Intuitively, the more pairs we bind together, the higher subgraph densities we have.
{Between minimal binding (i.e., EIGMs) and maximal binding that achieves the upper bound of subgraph densities, we can have a \textit{flexible} spectrum.}
However, the number of possible partitions of node pairs $\binom{V}{2}$ grows exponentially w.r.t. $\abs{V}$.
{Hence, we propose to introduce edge dependency \textit{without explicit partitions}.}
Specifically, we propose \textit{local binding},
where we repeatedly sample node groups,\footnote{We use independent \textit{node} sampling (yet still with \textit{edge} dependency), which is simple, tractable, and works empirically well in our experiments. See Appendix~C-C~\cite{appendix} for more discussions.\label{ftnt:indep_node_samp}} and
bind pairs between each sampled node group together.
Node pairs within the same node group are structurally related, and binding them together is expected to bring structurally meaningful edge dependency.

\smallsection{Real-world motivation.}
In social networks, each group ``bound together'' can represent a group interaction, e.g., an offline social event (meeting, conference, party) or an online social event (group chat, Internet forum, online game). 
In such social events, people gather together, and the communications/relations between them likely co-occur. 
At the same time, not all people in such events would necessarily communicate with each other, e.g., some people are more familiar with each other. This is the point of considering binding with various edge probabilities (instead of just inserting cliques).
{In general, group interactions widely exist in graphs in different domains, e.g., social networks~\citep{felmlee2013interaction}, biological networks~\citep{naoumkina2010genomic}, and web graphs~\citep{dourisboure2009extraction}.}
{{See Appendix~C-D~\cite{appendix} for more discussions.}}

In \cref{alg:local_binding}, we repeatedly sample a subset of nodes (Line~\ref{alg:local_binding:sample_nodes}) and group the ungrouped pairs between the sampled nodes (Line~\ref{alg:local_binding:sample_pairs}).
We maintain $P_{rem}$ to ensure disjoint partitions (Line~\ref{alg:local_binding:update_P_rem}).
For practical usage, we consider a limited number (i.e., $R$) of rounds for binding (Line~\ref{alg:local_binding:for_each_round}) otherwise it may take a long time to exhaust all the pairs.
Line~\ref{alg:local_binding} is also a probabilistic process, and we use $\flb$ to denote the corresponding RGM, i.e.,
$\flb(G) = \Pr[$\cref{alg:local_binding} outputs $G$ with inputs $p$, $g$, and $R$].
{As a special case of binding, local binding also correctly generates EPGMs.}

\begin{proposition}\label{thm:local_binding_EPPRGM}
    \cref{alg:local_binding} with input $p$ (and any $g$ and $R$) produces an EPGM w.r.t. $p$.
\end{proposition}
\begin{proof}[Proof sketch]
    It is a direct corollary of \cref{thm:binding_preserve_probs}.
\end{proof}
\begin{remark}\label{rem:local_binding_spectrum}
    We introduce node-sampling probabilities (i.e., $g$) to sample node groups with better tractability, without explicit partitions.
    With higher node-sampling probabilities, larger node groups are bound together, and the generated graphs are expected to have higher subgraph densities.
    Local binding forms a spectrum between the two extreme cases (see \cref{subsec:binding:binding}):
    $g(v) \equiv 0$ gives minimal binding and $g(v) \equiv 1$ gives maximal binding.
\end{remark}

\normalem
\begin{algorithm}[t!]
    \small
    \caption{Local binding}\label{alg:local_binding}
    \DontPrintSemicolon
    \SetKwInOut{Input}{Input}
    \SetKwInOut{Output}{Output}
    \SetKw{Continue}{continue}
    \SetKwComment{Comment}{\blue{$\triangleright$ }}{}
    \Input{
        (1) edge probabilities $\pfull$; \\
        \hspace{0.6pt} (2) node-sampling probabilities $g\colon V \to [0, 1]$; \\
        \hspace{0.6pt} (3) maximum number of rounds $R$
    }
    \Output{$G$: generated graph}
    $\calP \gets \emptyset; i_{round} \gets 0; P_{rem} \gets \binom{V}{2}$ \hfill \Comment{\blue{Initialization}}
    \label{alg:local_binding:init}
    \For{$i_{round} = 1, 2, \ldots, R$} {\label{alg:local_binding:for_each_round}
    \textbf{if} $P_{rem} = \emptyset$ \textbf{then} \textbf{break} \hfill \Comment{\blue{Pairs exhausted}}
    \label{alg:local_binding:pairs_exhausted}

    sample $V_s \subseteq V$ with $\Pr[v \in V_s] = g(v)$ independently \\
    \label{alg:local_binding:sample_nodes}
    $P_s \gets \binom{V_s}{2} \cap P_{rem}$ 
    \hfill \Comment{\blue{Group pairs}}
    \label{alg:local_binding:sample_pairs}
    $\calP \gets \calP \cup \set{P_s}$;
    $P_{rem} \gets P_{rem} \setminus P_s$ 
    \hfill \Comment{\blue{Update}}
    \label{alg:local_binding:update_P_rem}
    }    
    $\calP \gets \calP \cup \set{\set{(u, v)} \colon (u, v) \in P_{rem}}$ 
    \hfill \Comment{\blue{Remaining pairs}}
    \label{alg:local_binding:include_P_rem} 
    \Return{the output of \cref{alg:binding_general} with inputs $p$ and $\calP$ \label{alg:local_binding:generate_G}}
    \label{alg:local_binding:return_G}
\end{algorithm}
\ULforem

\begin{theorem}[Time complexity of graph generation with local binding]\label{thm:local_binding_time_gen}
    Given $\pfull$, $\gfull$, and $R \in \bbN$,
    $\flb$ generates a graph in $O(R \left(\sum_{v \in V} g(v)\right)^2 + \abs{V}^2)$ time with high probability,\footnote{That is, $\lim_{\abs{V} \to \infty} \Pr[\text{it takes $O(R \left(\sum_{v \in V} g(v)\right)^2 + \abs{V}^2)$}] = 1$.} with the worst case $O(R \abs{V}^2)$.
\end{theorem}
\begin{proof}[Proof sketch]
    In each round, it takes $O(n^2)$, where $n$ is the number of sampled nodes, which is $\sum_{v \in V} g(v)$ in expectation, thus
    $O(\sum_{v \in V} g(v))$ with high probability, and $O(|V|^2)$ at most.
    Dealing with the remaining pairs takes $O(\binom{\abs{V}}{2})$.    
\end{proof}

We derive tractability results of local binding on the closed-form expected number of motifs (i.e., induced subgraphs; see \cref{sec:prelim_background}).
For this, we derive the probabilities of all the possible size-$3$ motifs for each node group,
then we compute the expected number of each motif by taking the summation over all different node groups, which is later used for parameter fitting (see \cref{sec:fitting}).
\begin{theorem}[Tractable 3-motif probabilities with local binding]\label{thm:local_binding_motif_probs}
    For any $\pfull$, $\gfull$, $R \in \bbN$, any size-$3$ node group $V' = \set{u, v, w} \in \binom{V}{3}$, and any size-$3$ motif instance $E^* \subseteq \binom{V'}{2}$,
    we can compute the closed-form $\Pr_{\flb}[E(G[V']) = E^*]$, as a function w.r.t. $p$, $g$, and $R$ (the detailed formulae are in Appendix~A-D~\cite{appendix}).
\end{theorem}
\begin{proof}[Proof sketch]
    For each size-$3$ node group, we (1) consider all the possible cases of how the pairs among the three nodes are partitioned and grouped, (2) compute the motif probabilities conditioned on each sub-case, and 
    (3) take the summation of motif probabilities in all the sub-cases.
    See Appendix~A-D~\cite{appendix} for the full proof and detailed formulae.
\end{proof}
    Closed-form motif probabilities show the tractability of EPGMs with binding, which allows us to estimate the output and fit the parameters of RGMs (see \cref{sec:fitting}).
    Regarding the parameters, higher $p$ and $g$ give higher clustering, while the choice of $R$ is mainly for controlling the running time.

\begin{remark}\label{rem:local_binding_motif_prob_higher_order}    
    \cref{thm:local_binding_motif_probs} can be extended to larger motif sizes, with
    practical difficulties from the increasing sub-cases as motif size increases.
\end{remark}

\begin{theorem}[Time complexity of computing motif probabilities with local binding]\label{thm:local_binding_time}
    Computing $\Pr_{\flb}[E(G[V']) = E^*]$ takes $O(\abs{V}^3)$ time in total for all $E^* \subseteq \binom{V'}{2}$ and $V' \in \binom{V}{3}$.
\end{theorem}
\begin{proof}[Proof sketch]
    We enumerate all $O(|V|^3)$ size-$3$ node groups.
    For each group and each motif, the calculation takes $O(1)$ with fixed arithmetic formulae.
\end{proof}
\begin{remark}\label{rem:local_binding_time}
    The complexity can be reduced by considering node equivalence, which we will discuss in detail when discussing parameter fitting in \cref{sec:fitting}.
\end{remark}

\subsection{Parallel binding: Parallelizable icing on the cake}\label{sec:binding:para_binding}
In local binding, 
the sampling order matters, i.e., later rounds are affected by earlier rounds. 
If one pair is already determined in an early round, even if sampled again in later rounds, its (in)existence cannot be changed.
This property hinders the parallelization and the derivation of tractability analyses, and implies that each pair can only be bound together once, entailing less flexibility in group interactions.

We thus propose \textit{parallel binding}, a more flexible and naturally \textit{parallelizable} binding scheme.
The process is described in \cref{alg:iid_binding}.
We let $\fiid$ denote the corresponding RGM defined by $\fiid(G) = \Pr[\text{\cref{alg:iid_binding} outputs $G$ with inputs $p$, $g$, and $R$}]$.

The high-level idea is to make each round of binding probabilistically equivalent (see Lines \ref{alg:iid_binding:for_each_round} to \ref{alg:iid_binding:add_edges_each_round}).
In each round, we insert edges with some ``small'' probabilities $r$ calculated from $p$ (Line~\ref{alg:iid_binding:pair_prob_after_sampled}).
Those ``small'' probabilities accumulate over rounds, maintaining the final edge probabilities as $p$ (together with $p_{rem}$).
We can straightforwardly parallelize multiple rounds of binding by, e.g., multi-threading.

Although parallel binding is algorithmically different from (local) binding (e.g., no explicit partition is used), it shares many theoretical properties with local binding.
Specifically, Statements~\ref{thm:local_binding_EPPRGM} to \ref{thm:local_binding_time}
also apply to parallel binding.
This implies that we maintain (or even improve; see \cref{rem:para_binding_can_isolated_nodes}) correctness, tractability, flexibility, and efficiency when using parallel binding compared to local binding.
See Appendix~A-E~\cite{appendix} for the formal statements and proofs.

\begin{remark}\label{rem:para_binding_can_isolated_nodes}
We also derive tractability results of parallel binding on the expected number of (non-)isolated nodes.
It is much more challenging to derive such results for local binding where later rounds are affected by earlier rounds.
Since our main focus is on subgraph densities, 
see Appendix~B~\cite{appendix} for all the analysis regarding (non-)isolated nodes.
\end{remark}

\subsection{Efficient parameter fitting with node equivalence}\label{sec:fitting}

{With tractability results, we can fit parameters ($p$, $g$, and $R$) of binding (\cref{alg:local_binding,alg:iid_binding}) to the graph statistics of a given graph $G$, to generate random graphs that are similar to $G$, which is useful for the typical applications of RGMs, e.g., graph algorithm testing~\citep{mishra2011effective}, statistical testing~\citep{wandelt2019use}, and graph anonymization~\citep{casas2017survey}.}

{By default, we consider the cases where 
the number $R$ of rounds is fixed,
edge probabilities $p$ are obtained by applying an EIGM to the given graph $G$,
and we aim to fit the node-sampling probabilities $g$.
It is also possible to jointly optimize both $p$ and $g$ (see \cref{sec:exp:joint_optim}), or generate graphs ``from scratch'' with different levels of clustering by directly
setting the parameters (see Appendix~D-G~\cite{appendix}).}

Since our closed-form tractability results are functions involving only arithmetic operations on the parameters, differentiable optimization techniques (e.g., Newton-Raphson and gradient descent) are possible for parameter fitting.
The fitting objective can be any function of $3$-motif probabilities.
{For example, one can fit the (expected) number of triangles ($\bbE_{\fxx}[\triangle(G)]$) to the ground-truth number in an input graph ($\triangle(G_{input})$) by minimizing the ``difference'' between them, e.g., 
$(1 - \frac{\bbE_{\fxx}[\triangle(G)]}{\triangle(G_{input})})^2$,
where 
$\bbE_{\fxx}[\triangle(G)] = \sum_{V' \in \binom{V}{3}} \Pr_{\fxx}[E(G[V']) = \binom{V'}{2}]$ are obtainable by the tractability results (see \cref{thm:local_binding_motif_probs})
with $X \in \set{\text{LB}, \text{PB}}$ indicating the binding scheme (LB represents local binding and PB represents parallel binding).}

Efficient evaluation of the fitting objective is important.
{A key challenge is that the naive computation takes $O(\abs{V}^3)$ time in total by considering all $O(\abs{V}^3)$ different possible node groups $V'$ 
(see \cref{thm:local_binding_motif_probs}).}
We aim to improve the speed of computing the tractability results by considering \textit{node equivalence} and thus reducing the total number of distinct node groups we need to consider.

\smallsection{Erd\H{o}s-R\'{e}nyi (ER) model.}
The ER model~\citep{erdos1959randomgraph1} outputs edge probabilities with two parameters: $n_0$ and $p_0$,
and the output is $p^{ER}_{n_0, p_0}$ with $p^{ER}_{n_0, p_0}(u, v) = p_0, \forall u, v \in \binom{V}{2}$ with $V = [n_0]$.
Given a graph $G = (V = [n], E)$, ER outputs $n_0 = n$ and $p_0 = \frac{2\abs{E}}{n (n - 1)}$.
ER outputs uniform edge probabilities, and all nodes are equivalent. 
We use uniform node-sampling probabilities, i.e., $g(v) = g_0, \forall v \in V$ for a single parameter $g_0 \in [0, 1]$.

\begin{lemma}\label{lem:fitting_er_time}
    For ER, the time of computing 
    $3$-motif probabilities can be reduced 
    from $O(\abs{V}^3)$ to $O(1)$.
\end{lemma}

\normalem
\begin{algorithm}[t!]
    \caption{Parallel binding}\label{alg:iid_binding}
    \DontPrintSemicolon
    \SetKwInOut{Input}{Input}
    \SetKwInOut{Output}{Output}
    \SetKw{Continue}{continue}
    \SetKwComment{Comment}{\blue{$\triangleright$ }}{}
    \SetKwFunction{proc}{binding}
    \SetKwProg{myproc}{Procedure}{}{}
    
    \Input{
        (1) edge probabilities $p \colon \binom{V}{2} \to [0, 1]$\\ 
        \hspace{0.6pt} (2) node-sampling probabilities $g: V \to [0, 1]$\\ 
        \hspace{0.6pt} (3) the number of rounds $R$
    }
    \Output{$G$: generated graph}
    
    $E \gets \emptyset$ \label{alg:iid_binding:E_init} 
    \hfill \Comment{\blue{Initialization}}
    $r(u, v) \gets \min(\frac{1 - \left(1 - p(u, v)\right)^{1/R}}{g(u)g(v)}, 1), \forall u, v \in V$ \label{alg:iid_binding:pair_prob_after_sampled} \\
    $p_{rem}(u, v) \gets \max(1 - \frac{1 - p(u, v)}{(1 - g(u) g(v))^R}, 0), \forall u, v \in V$ \label{alg:iid_binding:pair_rem_prob} \\
    \For{$i_{round} = 1, 2, \ldots, R$} {\label{alg:iid_binding:for_each_round} 
        sample $V_s \subseteq V$ with $\Pr[v \in V_s] = g(v)$ independently \\ \label{alg:iid_binding:sample_nodes} 
        $E \gets E \cup \text{\proc{$r$, $\binom{V_s}{2}$}}$ \label{alg:iid_binding:add_edges_each_round}  \hfill  \Comment{\blue{See Alg. \ref{alg:binding_general}}}
    }
    sample a random variable $s \sim \calU(0, 1)$ \label{alg:iid_binding:s_RV_rem} \\
    $E \gets E \cup \{(u, v) : p_{rem}(u, v) \geq s\}$ \label{alg:iid_binding:edge_add_rem} \\
    \Return{$G = (V, E)$} \label{alg:iid_binding:return_G}

\end{algorithm}
\ULforem

\smallsection{Chung-Lu (CL) model.}
The CL model~\citep{chung2002connected} outputs edge probabilities with a sequence of expected degrees $D = (d_1, d_2, \ldots, d_n)$, and the output is $p^{CL}_{D}$ with $p^{CL}_{D}(u, v) = \min(\frac{d_u d_v}{\sum_{i = 1}^n d_i}, 1), \forall u, v \in \binom{V}{2}$ with $V = [n]$.
Given a graph $G = (V = [n], E)$, CL outputs $d_i = d(i; G)$ for each node $i \in V$.
CL outputs edge probabilities to match node degrees $D = (d_1, d_2, \ldots, d_n)$,
and nodes with the same degree are equivalent.
We set node-sampling probabilities as a function of degree (i.e., nodes with the same degree share the same node-sampling probability)
with $k_{deg}$ parameters,
where 
$k_{deg} \coloneqq \abs{\set{d_1, d_2, \ldots, d_n}}$ is the number of distinct degrees.

\begin{lemma}\label{lem:fitting_cl_time}
    For CL, the time of computing 
    $3$-motif probabilities
    can be reduced 
    from $O(\abs{V}^3)$ to $O(k^3_{deg})$.
\end{lemma}

\smallsection{Stochastic block (SB) model.}
Given a graph $G = (V = [n], E)$ and a node partition $f_B \colon [n] \to [c]$ with $c \in \bbN$, let 
$V_i = \set{v \in V \colon f_B(v) = i}$ denote the set of nodes partitioned in the $i$-th group for $i \in [c]$.   
The fitting of the edge probabilities in the SB~\citep{holland1983stochastic} model gives
$p_B \colon [c] \times [c] \to [0, 1]$ with
$p_B(i, i) = \frac{\abs{E(G[V_i])}}{\binom{\abs{V_i}}{2}}$ and
$p_B(i, j) = \frac{\abs{E \cap \set{(v, v') \colon v \in V_i, v' \in V_j}}}{\abs{V_i} \abs{V_j}}$,
for $i \neq j \in [c]$.
SB outputs edge probabilities with nodes assigned to different ``blocks'', and nodes partitioned in the same block are equivalent.
We set the node-sampling probabilities as a function of the block index, i.e., nodes in the same block share the same node-sampling probabilities, with the number of parameters equal to the total number $c$ of blocks.
\begin{lemma}\label{lem:fitting_SB_time}
    For SB, the time of computing 
    $3$-motif probabilities
    can be reduced 
    from $O(\abs{V}^3)$ to $O(c^3)$.
\end{lemma}

\smallsection{Stochastic Kronecker (KR) model.}
With a seed matrix $\theta \in [0,1]^{2 \times 2}$ and $k_{KR} \in \bbN$, the KR model~\citep{leskovec2010kronecker} outputs edge probabilities as the
$k_{KR}$-th Kronecker power of $\theta$ (see Appendix~A-F4~\cite{appendix} for a formal definition).
In practice, the fitting of KR typically uses \texttt{KronFit}~\citep{leskovec2010kronecker}, as proposed by the original authors of KR.
In KR, each node $i \in [2^{k_{KR}}]$ is associated with a binary vector of length $k_{KR}$, and nodes with the same number of ones in their vectors are equivalent.
We set node-sampling probabilities as a function of the number of ones in the binary vectors, with $k_{KR} + 1$ parameters, where $k_{KR} = O(\log \abs{V})$.
\begin{lemma}\label{lem:fitting_KR_time}
    For KR, the time of computing 
    $3$-motif probabilities
    can be reduced 
    from $O(\abs{V}^3)$ to $O(k^7_{KR})$.
\end{lemma}
\begin{remark}
    The equivalence in KR is weaker than that in the other three models.
This is why the reduced time complexity in \cref{lem:fitting_KR_time} is $O(k^7_{KR})$ instead of $O(k^3_{KR})$. See Appendix~A-F4~\cite{appendix} for more details.
\end{remark}

\smallsection{Note.}
See Appendix~A-F~\cite{appendix} for more details, e.g., formal definitions of the models and more details of node equivalence.

\section{Experiments}\label{sec:exp}

In this section, we empirically evaluate EPGMs with our binding schemes and show the superiority of realization schemes beyond edge independency.
Specifically, we show the following two points:
\begin{itemize}[noitemsep]
    \item \textbf{(P1)} When we use our tractability results to fit the parameters of EPGMs, we improve upon EIGMs and reproduce high triangle densities (high clustering), which is a common pattern in real-world graphs; this also validates the correctness of our tractability results and algorithms.
    \item \textbf{(P2)} We can reproduce other common patterns, e.g., power-law degrees and small diameters, especially when the corresponding EIGMs can do so; this shows that improving EIGMs w.r.t. clustering by binding does not harm the generation quality w.r.t. other common patterns.
\end{itemize}

\begin{table}[t!]
    \centering
    \caption{The basic statistics of the datasets, where $\triangle$ represents the number of triangles, GCC represents global clustering coefficient, and ALCC represents average local clustering coefficient (see \cref{sec:prelim_background}). The upscaled versions of \textit{Hams} are obtained by duplicating the original graph, which are used for scalability analysis (see \cref{sec:exp:speed}).}
    \scalebox{0.95}{
    \begin{tabular}{lrrrrr}
    \toprule
    dataset & $|V|$   & $|E|$   & $\triangle$ & GCC   & ALCC \\
    \midrule
    \hamster    & 2,000  & 16,097 & 157,953 & 0.229 & 0.540 \\
    \facebook    & 4,039  & 88,234 & 4,836,030 & 0.519 & 0.606 \\
    \polblog    & 1,222  & 16,717 & 303,129 & 0.226 & 0.320 \\
    \spam    & 4,767  & 37,375 & 387,051 & 0.145 & 0.286 \\
    \biopg    & 1,692  & 47,309 & 2,353,812 & 0.321 & 0.447 \\
    \bioht    & 2,077  & 63,023 & 4,192,980 & 0.377 & 0.350 \\
    \midrule
    \multicolumn{6}{l}{\textit{(Upscaled versions of \hamster used for scalability analysis)}} \\
    \multirow{2}{*}{Synthetic}  & 4,000 -   & 32,194 -  & 315,906 -  & \multirow{2}{*}{0.229} & \multirow{2}{*}{0.540} \\
    & 64,000,000 & 515,104,000 & 5,054,496,000 \\
    \bottomrule
    \end{tabular}%
    }
    \label{tab:datasets}
\end{table}

\begingroup
\setlength\tabcolsep{3pt}
\begin{table*}[t!]
    \centering    
    \caption{
The clustering metrics of generated graphs.
The number of triangles ($\triangle$) is normalized.
For each dataset and each model, the best result is in bold and the second best is underlined.
AR represents average ranking.
{
The statistics are averaged over 100 random trials.
See Table~VIII in Appendix~D-B~\cite{appendix} for the full results with standard deviations.}
\textbf{Our binding schemes (\LB and \PB) are consistently and clearly beneficial for improving clustering, and generating graphs with {close-to-ground-truth} clustering metrics.}}
\begin{adjustbox}{max width=\textwidth}
\begin{tabular}{c|c||ccc|ccc|ccc|ccc|ccc|ccc||ccc}
\hline
\multicolumn{2}{c||}{dataset} & \multicolumn{3}{c|}{\hamster} & \multicolumn{3}{c|}{\facebook} & \multicolumn{3}{c|}{\polblog} & \multicolumn{3}{c|}{\spam} & \multicolumn{3}{c|}{\biopg} & \multicolumn{3}{c||}{\bioht} & \multicolumn{3}{c}{AR over dataset} \bigstrut\\
\hline
\multicolumn{2}{c||}{metric} & $\triangle$     & \textsc{gcc}   & \small{\textsc{alcc}}  & $\triangle$     & \textsc{gcc}   & \small{\textsc{alcc}}  & $\triangle$     & \textsc{gcc}   & \small{\textsc{alcc}}  & $\triangle$     & \textsc{gcc}   & \small{\textsc{alcc}}  & $\triangle$     & \textsc{gcc}   & \small{\textsc{alcc}}  & $\triangle$     & \textsc{gcc}   & \small{\textsc{alcc}}  & $\triangle$     & \textsc{gcc}   & \small{\textsc{alcc}} \bigstrut\\
\hline
\hline
model & \GT    & 1.00 & 0.23 & 0.54 & 1.00 & 0.52 & 0.61 & 1.00 & 0.23 & 0.32 & 1.00 & 0.14 & 0.29 & 1.00 & 0.32 & 0.45 & 1.00 & 0.38 & 0.35 & N/A   & N/A   & N/A \bigstrut\\
\hline
\multirow{3}[2]{*}{ER} & \EI    & 0.01 & 0.01 & 0.01 & 0.01 & 0.01 & 0.01 & 0.03 & 0.02 & \underline{0.02} & 0.01 & \textbf{0.00} & \underline{0.00} & 0.04 & 0.03 & 0.03 & 0.03 & 0.03 & \underline{0.03} & 3.0   & 2.7   & 2.5 \bigstrut[t]\\
      & \LB    & \textbf{1.00} & \textbf{0.32} & \underline{0.24} & \underline{1.01} & \underline{0.45} & \underline{0.22} & \underline{0.95} & \textbf{0.34} & \textbf{0.25} & \underline{0.99} & \underline{0.34} & \textbf{0.23} & \textbf{1.02} & \textbf{0.40} & \textbf{0.26} & \underline{1.01} & \textbf{0.42} & \textbf{0.25} & \underline{1.7} & \textbf{1.3} & \textbf{1.3} \\
      & \PB    & \underline{0.99} & \underline{0.39} & \textbf{0.64} & \textbf{1.00} & \textbf{0.57} & \textbf{0.81} & \textbf{1.02} & \underline{0.41} & 0.66 & \textbf{0.99} & 0.40 & 0.66 & \underline{0.97} & \underline{0.51} & \underline{0.75} & \textbf{0.99} & \underline{0.56} & 0.79 & \textbf{1.3} & \underline{2.0} & \underline{2.2} \bigstrut[b]\\
\hline
\multirow{3}[2]{*}{CL} & \EI    & 0.30 & 0.07 & 0.06 & 0.12 & 0.06 & 0.06 & 0.79 & 0.18 & \underline{0.17} & 0.50 & 0.07 & 0.06 & 0.68 & 0.23 & 0.22 & 0.64 & 0.24 & \textbf{0.23} & 3.0   & 3.0   & 2.5 \bigstrut[t]\\
      & \LB    & \underline{0.99} & \underline{0.17} & \underline{0.26} & \underline{1.03} & \underline{0.26} & \underline{0.30} & \textbf{1.00} & \underline{0.21} & \textbf{0.34} & \underline{1.03} & \underline{0.12} & \textbf{0.26} & \underline{1.00} & \underline{0.29} & \textbf{0.43} & \textbf{1.04} & \textbf{0.32} & \underline{0.47} & \underline{1.7} & \underline{1.8} & \textbf{1.5} \\
      & \PB    & \textbf{1.00} & \textbf{0.18} & \textbf{0.47} & \textbf{1.01} & \textbf{0.34} & \textbf{0.63} & \underline{1.01} & \textbf{0.22} & 0.47 & \textbf{1.01} & \textbf{0.13} & \underline{0.44} & \textbf{1.00} & \textbf{0.31} & \underline{0.58} & \underline{1.14} & \underline{0.29} & 0.61 & \textbf{1.3} & \textbf{1.2} & \underline{2.0} \bigstrut[b]\\
\hline
\multirow{3}[2]{*}{SB} & \EI    & 0.26 & 0.08 & 0.04 & 0.15 & 0.14 & 0.08 & 0.48 & 0.14 & 0.16 & 0.53 & 0.09 & 0.04 & 0.66 & 0.26 & 0.20 & 0.64 & 0.27 & 0.13 & 3.0   & 3.0   & 3.0 \bigstrut[t]\\
      & \LB    & \underline{1.04} & \textbf{0.22} & \underline{0.24} & \underline{0.93} & \underline{0.43} & \underline{0.33} & \textbf{0.99} & \textbf{0.24} & \textbf{0.35} & \underline{0.98} & \textbf{0.15} & \textbf{0.22} & \textbf{0.99} & \textbf{0.32} & \textbf{0.41} & \underline{1.03} & \textbf{0.35} & \textbf{0.39} & \underline{1.7} & \textbf{1.2} & \textbf{1.3} \\
      & \PB    & \textbf{0.99} & \underline{0.24} & \textbf{0.52} & \textbf{1.03} & \textbf{0.53} & \textbf{0.56} & \underline{1.01} & \underline{0.18} & \underline{0.25} & \textbf{0.99} & \underline{0.16} & \underline{0.36} & \underline{1.05} & \underline{0.33} & \underline{0.36} & \textbf{0.97} & \underline{0.34} & \underline{0.44} & \textbf{1.3} & \underline{1.8} & \underline{1.7} \bigstrut[b]\\
\hline
\multirow{3}[2]{*}{KR} & \EI    & 0.18 & 0.04 & 0.06 & 0.05 & 0.04 & 0.04 & 0.10 & 0.04 & 0.07 & 0.06 & 0.01 & 0.03 & 0.13 & 0.07 & 0.12 & 0.03 & 0.03 & 0.05 & 3.0   & 3.0   & 3.0 \bigstrut[t]\\
      & \LB    & \underline{1.09} & \underline{0.15} & \underline{0.23} & \underline{0.93} & \underline{0.24} & \underline{0.27} & \underline{1.06} & \underline{0.14} & \textbf{0.23} & \underline{0.94} & \underline{0.12} & \underline{0.19} & \underline{0.99} & \underline{0.17} & \underline{0.31} & \underline{1.44} & \underline{0.18} & \textbf{0.28} & \underline{2.0} & \underline{2.0} & \underline{1.7} \\
      & \PB    & \textbf{1.00} & \textbf{0.17} & \textbf{0.39} & \textbf{0.97} & \textbf{0.35} & \textbf{0.60} & \textbf{0.94} & \textbf{0.22} & \underline{0.42} & \textbf{1.05} & \textbf{0.16} & \textbf{0.38} & \textbf{1.00} & \textbf{0.28} & \textbf{0.46} & \textbf{1.07} & \textbf{0.35} & \underline{0.58} & \textbf{1.0} & \textbf{1.0} & \textbf{1.3} \bigstrut[b]\\
\hline
\hline
\multirow{3}[2]{*}{\shortstack[c]{AR\\over\\models}} & \EI    & 3.0   & 3.0   & 3.0   & 3.0   & 3.0   & 3.0   & 3.0   & 3.0   & \underline{2.5} & 3.0   & 2.5   & 2.8   & 3.0   & 3.0   & 3.0   & 3.0   & 3.0   & \underline{2.3} & 3.0   & 2.9   & 2.8 \bigstrut[t]\\
      & \LB    & \underline{1.8} & \textbf{1.5} & \underline{2.0} & \underline{2.0} & \underline{2.0} & \underline{2.0} & \textbf{1.5} & \textbf{1.5} & \textbf{1.0} & \underline{2.0} & \textbf{1.8} & \textbf{1.3} & \textbf{1.5} & \textbf{1.5} & \textbf{1.3} & \underline{1.8} & \textbf{1.3} & \textbf{1.3} & \underline{1.8} & \underline{1.6} & \textbf{1.5} \\
      & \PB    & \textbf{1.3} & \textbf{1.5} & \textbf{1.0} & \textbf{1.0} & \textbf{1.0} & \textbf{1.0} & \textbf{1.5} & \textbf{1.5} & \underline{2.5} & \textbf{1.0} & \textbf{1.8} & \underline{2.0} & \textbf{1.5} & \textbf{1.5} & \underline{1.8} & \textbf{1.3} & \underline{1.8} & 2.5   & \textbf{1.3} & \textbf{1.5} & \underline{1.8} \bigstrut[b]\\
\hline
\end{tabular}%
\end{adjustbox}
\label{tab:main_cluster}
\vspace{-3mm}
\end{table*}

\endgroup

\subsection{Experimental settings}
\smallsection{Datasets.}
We use six real-world datasets:
(1) social networks 
\textit{hamsterster (\hamster)} and \textit{facebook (\facebook)},
(2) web graphs \textit{polblogs (\polblog)} and
\textit{spam (\spam)}, and
(3) biological graphs \textit{CE-PG (\biopg)} and
\textit{SC-HT (\bioht)}.
{See \cref{tab:datasets} for the basic statistics of the datasets.}
We also upscale \textit{Hams} by duplicating the original graph for scalability analysis (see \cref{sec:exp:speed}).

\smallsection{Models.}
We consider the four edge-probability models analyzed in \cref{sec:fitting}:
the Erd\H{o}s-R\'{e}nyi (ER) model,
the Chung-Lu (CL) model,
the stochastic block (SB) model, and
the stochastic Kronecker (KR) model.
{Given an input graph, we fit each model to the graph and obtain the output edge probabilities (see \cref{sec:fitting} and Appendix~A-F~\cite{appendix} for more details).}

\smallsection{Realization methods.}
We compare edge-probability graph models with three different realization methods:
EIGMs (\EI) assuming edge dependency, and 
EPGMs with the two proposed binding schemes: local binding (\LB) and parallel binding (\PB).

\begin{figure*}[t!]
    \centering  
    \includegraphics[width=0.5\linewidth]{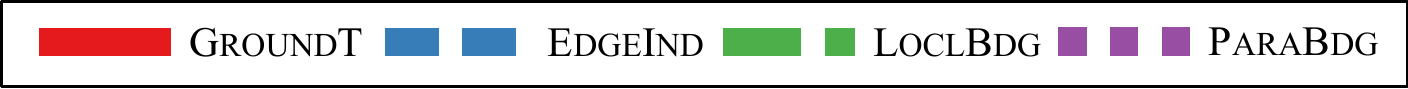} \\
    \includegraphics[width=\linewidth]{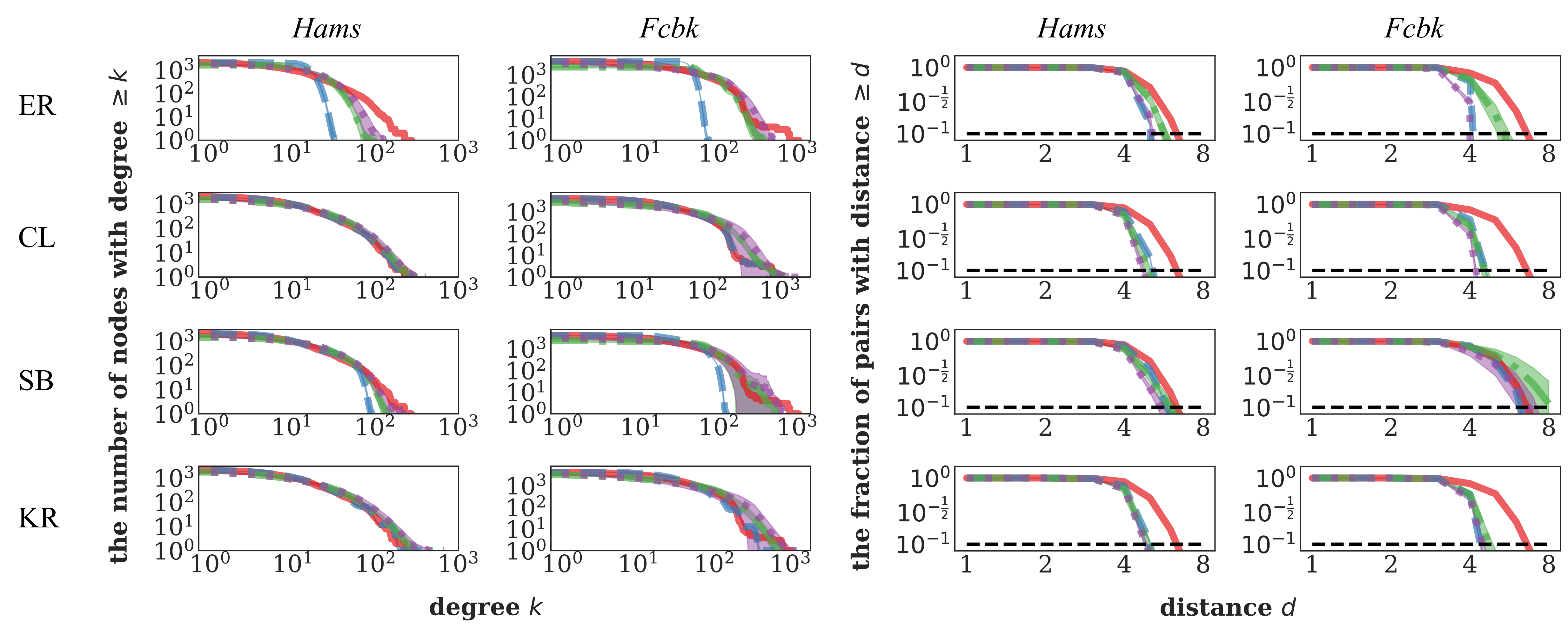}
    \vspace{-3mm}
    \caption{The degree (left) and distance (right) distributions
    of generated graphs.
    Each shaded area represents one standard deviation.
    The plots are on a log-log scale.
    \textbf{Our binding schemes (\LB and \PB) do not negatively affect degree or distance distributions, and provide improvements sometimes (e.g., for the ER model).}}
    \vspace{-5mm}
    \label{fig:main_fig}
\end{figure*}

\smallsection{Fitting.}
As in \cref{sec:fitting}, by default, we use the edge probabilities $p$ provided by the models, and we
only fit the node-sampling probabilities $g$.
Since our main focus is to improve clustering (see our motivation for binding at the start of \cref{sec:binding}), in our main experiments, we use the number of triangles, {an important indicator of clustering}~\citep{tsourakakis2009doulion,kolda2014counting}, as the objective of the fitting algorithms.
We use gradient descent to optimize parameters.
{In the main experiments, the edge probabilities are fixed as those output by the edge-probability models, while we also consider joint optimization of edge probabilities and node-sampling probabilities (see \cref{sec:exp:joint_optim}).}
Instead of fitting specific graphs, it is also possible to use EPGMs with binding to generate graphs ``from scratch'' with different levels of clustering by directly setting the parameters (see Appendix~D-G~\cite{appendix}).

\smallsection{Hardware and software.}
All the experiments of fitting are run on a machine with
two Intel Xeon\textsuperscript{\textregistered} Silver 4210R (10 cores, 20 threads) processors,
a 512GB RAM,
and RTX A6000 (48GB) GPUs.
A single GPU is used for each fitting process.
The code for fitting is written in Python, using Pytorch~\citep{paszke2019pytorch}.
All the experiments of graph generation are run on a machine with
one Intel i9-10900K (10 cores, 20 threads) processor,
a 64GB RAM.
The code for generation is written in C++, compiled with G++ with O2 optimization and OpenMP~\citep{dagum1998openmp} parallelization.
{See Appendix~D-A~\cite{appendix} for the detailed experimental settings.}

\subsection{P1: EPGMs reproduce high clustering (\cref{tab:main_cluster})}\label{sec:app:p1_clustering}
EPGMs with binding reproduce high clustering {in real-world graphs}.
In \cref{tab:main_cluster}, for each dataset and each model, we compare three clustering-related metrics, the number of triangles ($\triangle$), the global clustering coefficient (GCC), and the average local clustering coefficient (ALCC), in the ground-truth (\GT) graph and the graphs generated with each realization method.
For each dataset and each model, we compute the ranking of each method according to the absolute error w.r.t. each metric.
We also show the average rankings (ARs) over datasets and models.
The statistics are averaged over $100$ generated graphs.
{See Appendix~D-B~\cite{appendix} for the full results with standard deviations.}

EPGMs with both \LB and \PB almost perfectly reproduce the number of triangles in real-world graphs, while EIGMs (i.e., \EI in \cref{tab:main_cluster}) often fail to generate graphs with enough triangles.
GCC and ALCC are also significantly improved (upon EIGMs) in most cases, while
\PB has noticeably higher ALCC than \LB.
In some rare cases, \PB generates graphs with exceedingly high GCC and/or ALCC and haves higher absolute errors compared to EIGMs.

\begin{figure}[t!]
    \centering  
    \includegraphics[width=0.75\linewidth]{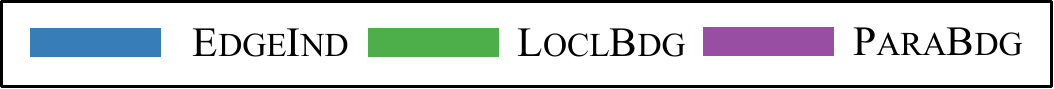} \\
    \includegraphics[width=0.9\linewidth]{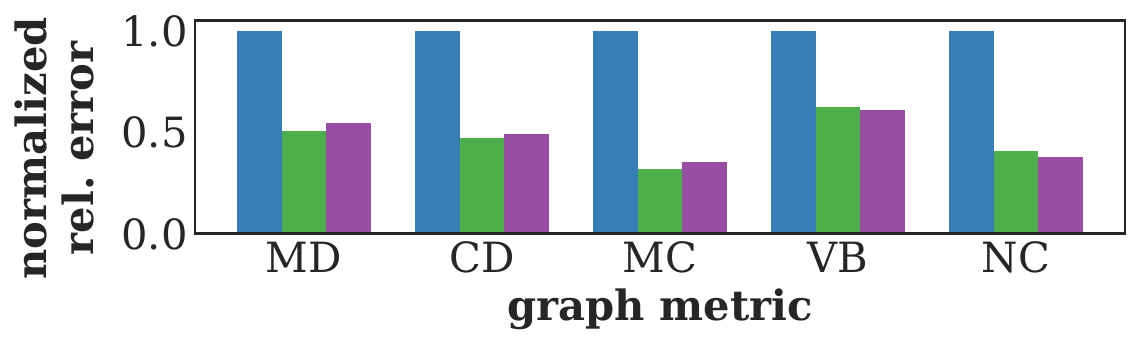}
    \caption{The normalized relative error w.r.t. different graph metrics of generated graphs. 
    For each graph metric: modularity (MD), conductance (CD), max core numbers (MC),  vertex betweenness (VB), and natural connectivity (NC), the relative errors are computed and averaged over all datasets and all models, and normalized so that the relative error of \EI is $1$. See Appendix~D-C~\cite{appendix} for the detailed results.  
    \textbf{Our binding schemes (\LB and \PB) generate graphs that are closer to the ground truth w.r.t. various graph metrics, consistently improving upon \EI.}}
    \label{fig:metrics_summary}
\end{figure}

\subsection{P2: EPGMs {reproduce other patterns} (Figures~\ref{fig:main_fig} and \ref{fig:metrics_summary})}\label{subsec:EPGMs_do_not_harm_others}
EPGMs with binding (\LB and \PB) also reproduce other {common patterns in real-world graphs}.
In \cref{fig:main_fig}, for each dataset (each column) and each model (each row), we compare the degree distributions and distance distributions in the ground-truth graph and the graphs generated with each realization method.
Specifically, for each realization method, we count the number of nodes with degree at least $k$ for each $k \in \bbN$ 
and count the number of pairs in the largest connected component with distance at least $d$ for each $d \in \bbN$ in each generated graph, and take the average number over $100$ generated graphs.
{See Appendix~D-C~\citep{appendix} for the full results.}

EPGMs with binding generate graphs with {common} patterns: power-law degrees and small diameters (i.e., small distances).
Both schemes (\LB and \PB) perform comparably well while \LB performs noticeably better with ER and \PB performs noticeably better with KR.
Importantly, when the edge probabilities output power-law expected degrees (e.g., CL and KR), the degree distributions are well preserved with binding.
Edge-independent ER cannot generate power-law degrees~\citep{bollobas2003mathematical}, and binding alleviates this problem.
Also, it is known that oscillations exist in the degree distributions of graphs generated by edge-independent KR~\citep{seshadhri2013depth,pinar2012similarity}, which is also alleviated by binding.

\smallsection{Other graph metrics.}
With binding, the generated graphs are also closer to the ground truth w.r.t. some other graph metrics: modularity, core numbers, conductance, vertex betweenness, and natural connectivity.
See Figure~\ref{fig:metrics_summary} for a summary of the results, and see Appendix~D-C~\cite{appendix} for the detailed results.

\subsection{Graph generation speed and scalability (Tables \ref{wrap-tab:time} and \ref{tab:upscale_main})}\label{sec:exp:speed}

\setlength{\intextsep}{0pt}%
\setlength{\columnsep}{10pt}%
\begin{wraptable}{r}{4cm}
\caption{The time (in seconds) for graph generation with different realization methods.}\label{wrap-tab:time}
\begin{adjustbox}{max width=\linewidth}
\begin{tabular}{lrr}
\toprule
      & \multicolumn{1}{l}{\hamster} & \multicolumn{1}{l}{\facebook} \\
\midrule
\EI   & 0.1   & 0.1 \\
\LB   & 4.7   & 49.0 \\
\PB   & 0.3   & 1.7 \\
\PB-\textsc{s} & 3.2   & 12.6 \\
\bottomrule
\end{tabular}%
\end{adjustbox}
\end{wraptable}
In Table~\ref{wrap-tab:time}, we compare the running time of graph generation (averaged over $100$ random trials) using \EI,
\LB, \PB, and serialized \PB without parallelization (\PB-s) with the stochastic Kronecker (KR) model.
For \EI, we use \texttt{krongen}~\citep{leskovec2016snap}, which is an algorithm parallelized and optimized for KR.

Among the competitors, \EI is the fastest with the simplest algorithmic nature.
Between the two binding schemes,
\PB is noticeably faster than \LB, and is even faster with parallelization.
Fitting the same number of triangles, \PB usually requires lower node-sampling probabilities and thus deals with fewer pairs in each round, and is thus faster even when serialized.

We also conduct scalability experiments by upscaling the \textit{Hams} dataset (see \cref{tab:datasets}).
As shown in \cref{tab:upscale_main}, with 32GB RAM, all the proposed algorithms can run with 128,000 nodes.
Notably, for \PB, an optimized implementation scales to up to 64 million nodes (over 512 million edges), by storing compact summaries (e.g., cliques and bi-cliques) instead of individual edges.
{See Appendix~D-D~\cite{appendix} for more details.}

\begin{table}[t!]
    \centering    
    \caption{The time (in seconds) for graph generation when upscaling the \textit{Hams} dataset by duplicating the graph.
    \textbf{With 32GB
RAM, all the proposed algorithms can run with 128,000 nodes,
and the scalable implementation of \PB can run with $64$ million nodes (over $512$ million edges).
}
}
\begin{adjustbox}{max width=\linewidth}
\begin{tabular}{clrrrrrrr}
\toprule
model  & $|V|$   & 2k    & 4k    & 8k    & 16k   & 32k   & 64k   & 128k \\
\midrule
\multirow{2}[0]{*}{ER} & \LB & 3.2   & 6.5   & 16.4  & 45.6  & 143.4 & 494.5 & 1859.2 \\
      & \PB  & 0.0   & 0.1   & 0.1   & 0.2   & 0.6   & 1.7   & 5.4   \\
\midrule
\multirow{2}[0]{*}{CL} & \LB & 4.0   & 9.6   & 35.4  & 123.9 & 472.3 & 2162.3& 8402.2 \\
      & \PB  & 0.3   & 0.5   & 1.0   & 2.1   & 4.4   & 11.2  & 31.1  \\
\midrule
\multirow{2}[0]{*}{SB} & \LB & 4.0   & 9.5   & 29.6  & 99.2  & 362.9 & 1648.4& 8398.1 \\
      & \PB  & 0.3   & 0.5   & 1.0   & 2.1   & 5.3   & 14.9  & 46.0  \\
\midrule
\multirow{2}[0]{*}{KR} & \LB & 8.6   & 31.2  & 124.5 & 506.9 & 2097.2& 8681.0& 33918.4\\
      & \PB  & 0.4   & 1.2   & 4.3   & 20.3  & 113.5 & 705.6 & 4351.6\\
\midrule\midrule
model  & $|V|$   & 1m    & 2m    & 4m    & 8m   & 16m   & 32m   & 64m \\
\midrule
ER & \multirowcell{6}{\PB \\ (scalable \\ version)} & 5.9 & 12.4 & 28.2 & 61.0 & 121.9 & 262.7 & 491.0 \\
\cmidrule(lr){1-1} \cmidrule(lr){3-9}
CL & & 102.2 & 220.2 & 423.8 & 815.9 & 1685.6 & 3135.2 & 6179.3 \\
\cmidrule(lr){1-1} \cmidrule(lr){3-9}
SB & & 106.0 & 213.7 & 429.0 & 869.0 & 1798.3 & 3829.6 & 8639.0 \\
\cmidrule(lr){1-1} \cmidrule(lr){3-9}
KR & & 105.1 & 219.4 & 439.1 & 875.4 & 1751.3 & 3504.7 & 7014.9 \\
\bottomrule
\end{tabular}%
\end{adjustbox}
\label{tab:upscale_main}
\end{table}

\subsection{Joint optimization (Table~\ref{tab:joint_optim})}\label{sec:exp:joint_optim}

In addition to optimizing node-sampling probabilities for given edge probabilities, we can also jointly optimize both kinds of probabilities.
In Table~\ref{tab:joint_optim},
we compare the ground-truth clustering and
that generated by EPGMs using three variants of parallel binding:
(1) \PB with the number of triangles as the objective 
(the one used in \cref{tab:main_cluster}),
(2) \PB-\textsc{w} with the numbers of triangles and wedges as the objective (given edge probabilities),
(3) \PB-\textsc{jw} jointly optimizing both kinds of probabilities, with the numbers of triangles and wedges as the objective.

\begingroup
\setlength\tabcolsep{7pt}
\begin{table}[t!]
\vspace{3mm}
\caption{
    The clustering metrics of the graphs generated by three variants of parallel binding.
    The number of triangles ($\triangle$) is normalized.
    For each dataset and each model, the best result is in bold, and the second best is underlined.
    \textbf{Using joint optimization further enhances the power of our binding scheme \PB to fit graph statistics and reproduce graph patterns.}
    }\label{tab:joint_optim}    
    \centering    
    \begin{adjustbox}{max width=\linewidth}
    \begin{tabular}{l||ccc||ccc}
    \hline
    dataset & \multicolumn{3}{c||}{\hamster}       & \multicolumn{3}{c}{\facebook} \bigstrut\\
    \hline
    \hline
    metric      & $\triangle$     & GCC   & ALCC       & $\triangle$     & GCC   & ALCC \bigstrut\\
    \hline
    \hline
    \GT    & 1.000 & 0.229 & 0.540 & 1.000 & 0.519 & 0.606 \bigstrut\\
    \hline
    \PB & 
    \underline{0.997} & 0.165 & \underline{0.394} & 
    0.971 & 0.347 & \textbf{0.605} \bigstrut[t]\\
    \PB-\textsc{w} 
    & 0.964 & \underline{0.176} & 0.260 & \underline{1.021} & \underline{0.408} & 0.458 \\
    \PB-\textsc{jw} & 
    \textbf{0.999} & \textbf{0.230} & \textbf{0.448} & 
    \textbf{1.018} & \textbf{0.521} & \underline{0.644} \bigstrut[b]\\
    \hline
    \end{tabular}%
    \end{adjustbox}
\end{table}

\endgroup

On both \hamster and \facebook, \PB and \PB-\textsc{w} can well fit the number of triangles but have noticeable errors w.r.t. the number of wedges (and thus GCC), while \PB-\textsc{jw} with joint optimization accurately fits both triangles and wedges.
On the other datasets, the three variants perform similarly well because \PB already preserves both triangles and wedges well, and there is not much room for improvement.
Notably, with joint optimization, the degree and distance distributions are still well preserved (see Appendix~D-E~\cite{appendix}).

\begin{figure}[t!]
    \centering
    \includegraphics[width=\linewidth]{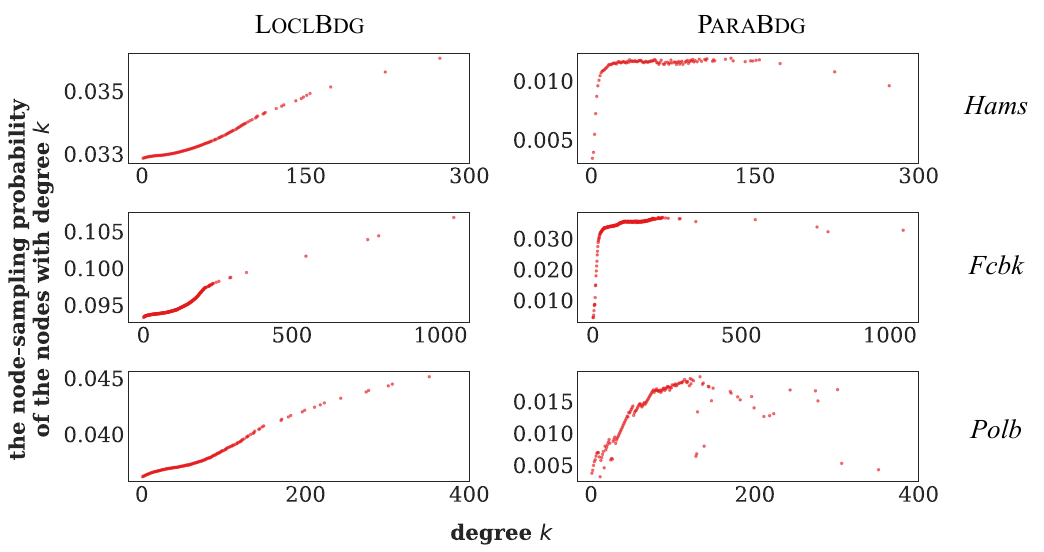}
    \caption{The relations between node degrees and node-sampling probabilities. See Appendix~D-A~\citep{appendix} for the full results. }\label{fig:corr_deg_alpha_main}
\end{figure}

\subsection{Comparison with advanced RGMs (\cref{tab:bter_main_paper})}\label{sec:compare_other_models}

We test 
other edge-dependent RGMs: preferential attachment (PA) model~\citep{barabasi1999emergence}) and
random geometric graphs (RGG) model~\citep{penrose2003random}), 
and advanced EIGMs: block two-level Erd\H{o}s-R\'{e}nyi (BTER) model~\citep{kolda2014scalable} and Lancichinetti-Fortunato-Radicchi (LFR) model~\citep{lancichinetti2008benchmark}.
In \cref{tab:bter_main_paper}, we report clustering-related metrics and overlap (low overlap indicates high output variability; see \cref{sec:rel_wk:limit_eigms}) of the graphs generated by each model for \hamster and \facebook datasets.
See Appendix~D-F~\cite{appendix} for full results on other datasets.

We fit PA and RGG to the numbers of nodes and edges of each input graph.
As discussed in \cref{sec:rel_wk:dependency}, closed-form tractability results on subgraph densities are not unavailable for PA and RGG.
PA fails to generate graphs with high clustering.
For RGG, sometimes enough triangles are generated with small $d$ values, but the GCC and ALCC are too high in such cases.
We also tried exchangeable network models and exponential random graphs, which failed to generate comparable results.
See Appendix~D-B~\cite{appendix} for more details.

BTER essentially uses a mixture of multiple Chung-Lu models to generate high clustering, and LFR imposes community structures on top of Chung-Lu.
We empirically validate that, compared to EIGMs, EPGMs with binding (we report the results based on Chung-Lu; one may achieve even better performance with binding based on other edge-probability models, as shown in \cref{tab:main_cluster}) achieve comparable or better performance in generating high-clustering graphs, with much higher output variability (i.e., low overlap), which is important and desirable for RGMs (see \cref{sec:rel_wk}).

{In Appendix~D-F~\cite{appendix}, we also evaluate three deep graph generative models: CELL~\citep{rendsburg2020netgan},
GraphVAE~\citep{simonovsky2018graphvae}, and
GraphRNN~\citep{you2018graphrnn}.
All of them failed to generate comparable results: CELL and GraphVAE produce highly overlapping graphs (i.e., low output variability), while GraphRNN generates graphs with many more edges but much lower clustering than the ground truth.}

\subsection{Extra experimental results}\label{sec:extra_exp}

See Figure~\ref{fig:corr_deg_alpha_main} for analyses on node-sampling probabilities assigned by our fitting algorithms. We observe that our fitting algorithms assign different node-sampling probabilities to different nodes, implying that different nodes have different levels of importance in binding. Full results are in Appendix~D~\cite{appendix}.
See Appendix~B~\cite{appendix} for the theoretical analyses and experimental results on fitting the number of (non-)isolated nodes, which further shows the tractability of our binding schemes.

\begin{table}[t!]
\setlength\tabcolsep{3pt}
\caption{
    The clustering metrics and overlap (lower the better) of the graphs generated by binding and other models.
    For each dataset and each model, the best result is in bold, and the second best is underlined.
    \textbf{Overall, binding achieves promising performance in generating high-clustering graphs, with high variability.}
    }\label{tab:bter_main_paper}    
    \centering    
    \begin{adjustbox}{max width=\linewidth}
    \begin{tabular}{l||cccc||cccc}
    \hline
    dataset & \multicolumn{4}{c||}{\hamster}       & \multicolumn{4}{c}{\facebook} \bigstrut\\
    \hline
    \hline
    metric      
    & $\triangle$     & GCC   & ALCC  & overlap      
    & $\triangle$     & GCC   & ALCC  & overlap \bigstrut\\
    \hline
    \hline
    \GT    
    & 1.000 & 0.229 & 0.540 & N/A
    & 1.000 & 0.519 & 0.606 & N/A
    \bigstrut\\
    \hline    
    \LB-CL
    & \underline{0.992} & 0.165 & 0.255 & {5.8\%}
    & \underline{1.026} & 0.255 & 0.305 & {6.3\%}
    \bigstrut[t] \\
    \PB-CL
    & \textbf{1.000} & \textbf{0.185} & 0.471 & 5.9\%
    & \textbf{1.006} & 0.336 & 0.626 & {6.2\%}
    \\
    PA
    & 0.198 & 0.049 & 0.049 & {4.7}\%
    & 0.120 & 0.061 & 0.061 & 6.2\%
    \\
    RGG ($d = 1$)
    & 1.252 & 0.751 & 0.751  & \textbf{0.8}\%
    & 0.607 & 0.751 & 0.752  & \textbf{1.1}\%
    \\
    RGG ($d = 2$)
    & 1.011 & 0.595 & 0.604  & \underline{0.8}\%
    & 0.492 & 0.596 & \underline{0.607}  & \underline{1.1}\%
    \\
    RGG ($d = 3$)
    & 0.856 & 0.491 & \underline{0.513}  & 0.8\%
    & 0.421 & \underline{0.494} & 0.518  & 1.1\%
    \\
    BTER
    & 0.991 & \underline{0.290} & \textbf{0.558} & 53.8\%
    & 0.880 & \textbf{0.525} & \textbf{0.605} & 68.0\%
    \\
    LFR ($\mu = 0.0$)
    & 1.140 & 0.262 & 0.546 & 43.5\%
    & N/A & N/A & N/A & N/A
    \\
    LFR ($\mu = 0.5$)
    & 0.296 & 0.068 & 0.081 & 13.4\%
    & 0.161 & 0.084 & 0.120 & 17.0\%
    \\    
    LFR ($\mu = 1.0$)
    & 0.197 & 0.045 & 0.047 & 7.0\%
    & 0.105 & 0.055 & 0.059 & 6.7\% 
    \bigstrut[b] \\
    \hline
    \end{tabular}%
    \end{adjustbox}
\end{table}

\section{Conclusion and discussions}\label{sec:conc_and_limit}

We show that realization beyond edge independence can 
{better reproduce common patterns}
while ensuring high tractability and variability.
We formally define EPGMs and show their basic properties (\cref{sec:prob_statement}).
We propose a {pattern-reproducing}, tractable, and flexible realization framework called \textit{binding} (\cref{alg:binding_general}) with two practical variants: local binding (\cref{alg:local_binding}) and parallel binding (\cref{alg:iid_binding}).
We derive tractability results (\cref{thm:local_binding_motif_probs}) on the closed-form subgraph densities, and propose efficient parameter fitting (\cref{sec:fitting}; Lemmas~\ref{lem:fitting_er_time}-\ref{lem:fitting_KR_time}).
We conduct extensive experiments to show the empirical power of EPGMs with binding (\cref{sec:exp}). 

\smallsection{Limitations and future directions.}
EPGMs with binding generate more isolated nodes than EIGMs due to higher variance.
Fortunately, we can address the limitation by 
fitting and controlling the number of isolated nodes with the tractability results, as mentioned in \cref{rem:para_binding_can_isolated_nodes}.

{The performance of EPGMs depends on both the underlying edge probabilities and the way to realize (i.e., sample from) them. In this work, we focus on the latter, while finding valuable edge probabilities is an independent problem. Notably, as shown in \cref{sec:exp:joint_optim}, it is possible to jointly optimize both edge probabilities and their realization.}

Binding only covers a subset of EPGMs, and we will explore the other types of EPGMs (e.g., EPGMs with lower subgraph densities).
Combining binding with other mechanisms in existing edge-dependent RGMs to create even stronger RGMs, and developing practical algorithms for higher-order motifs, are some interesting future directions.

\section*{Acknowledgments}
{\small This work was partly supported by the National Research Foundation of Korea (NRF) grant funded by the Korea government (MSIT) (No. RS-2024-00406985).
This work was partly supported by Institute of Information \& Communications Technology Planning \& Evaluation (IITP) grant funded by the Korea government (MSIT) (No. RS-2024-00438638, EntireDB2AI: Foundations and Software for Comprehensive Deep Representation Learning and Prediction on Entire Relational Databases) (No. RS-2019-II190075, Artificial Intelligence Graduate School Program (KAIST)). 
F.B. gives special thanks to Prof. Jaehoon Kim (KAIST), from whom F.B. studied the probabilistic method.
R.Y. and P.B acknowledge the support by the U.S. ArmyResearch Office (ARO) under Grant No. W911NF-23-1-0111, National Science Foundation (NSF) under the Career Award CPS-1453860, CNS-1932620 and award No. 2243104, Center for Complex Particle Systems (COMPASS), DARPA Young Faculty Award and DARPA Director Award under Grant Number N66001-17-1-4044. P.B. is also grateful to National Institute of Health (NIH) for the grants R01 AG 079957 ``Interpretable machine learning to synergize brain age estimation and neuroimaging genetics'' and RF1 AG 082201 ``Neurovascular calcification and ADRD in two nonindustrial Native American populations''.  The views, opinions, and/or findings in this article are those of the authors and should not be interpreted as official views or policies of Department of Defense, National Science Foundation, or National Institutes of Health.}

\normalem

\bibliographystyle{IEEEtran}
\bibliography{ref}

\end{document}